\documentclass{article}

\usepackage{microtype}
\usepackage{graphicx}
\usepackage{subfigure}
\usepackage{booktabs} 

\usepackage{hyperref}

\usepackage[accepted]{icml2024}

\usepackage{amsmath}
\usepackage{amssymb}
\usepackage{mathtools}
\usepackage{amsthm}
\usepackage{thmtools}
\usepackage{todonotes}

\usepackage[capitalize,noabbrev]{cleveref}

\theoremstyle{plain}
\newtheorem{theorem}{Theorem}[section]

\newtheorem{lemma}[theorem]{Lemma}

\theoremstyle{definition}
\newtheorem{definition}[theorem]{Definition}
\newtheorem{assumption}[theorem]{Assumption}
\theoremstyle{remark}

\DeclarePairedDelimiter{\norm}{\|}{\|}

\DeclareMathOperator*{\E}{\mathbb{E}}

\providecommand{\norm}[1]{\lVert#1\rVert}

\newcommand{\cA}{\mathcal{A}}

\newcommand{\cC}{\mathcal{C}}

\newcommand{\cI}{\mathcal{I}}

\newcommand{\cP}{\mathcal{P}}

\newcommand{\bu}{{\mathbf u}}
\newcommand{\bv}{{\mathbf v}}

\newcommand{\bx}{{\mathbf x}}
\newcommand{\by}{{\mathbf y}}

\newcommand{\ignore}[1]{}

\newenvironment{proof*}{\trivlist\item[\hskip\labelsep{\it Proof}{.}]}

\newcommand{\TV}{\mathrm{TV}}
\newcommand{\KL}{\mathrm{KL}}

\newcommand{\shortleft}{\scriptscriptstyle\leftarrow}
\newcommand{\R}{\mathbb{R}}
\newcommand{\D}{\mathrm{d}}

\newcommand{\forward}{X}
\newcommand{\reverse}{X^{\shortleft}}

\usepackage{amsmath,amssymb,bm}

\icmltitlerunning{Bias Begets Bias: the Impact of Biased Embeddings on Diffusion Models}

\begin{document}

\twocolumn[
\icmltitle{Bias Begets Bias: the Impact of Biased Embeddings on Diffusion Models}



\icmlsetsymbol{equal}{*}

\begin{icmlauthorlist}
\icmlauthor{Sahil Kuchlous}{equal,yyy}
\icmlauthor{Marvin Li}{equal,yyy}
\icmlauthor{Jeffrey G. Wang}{equal,yyy}
\end{icmlauthorlist}

\icmlaffiliation{yyy}{Harvard School of Engineering and Applied Sciences, Boston, USA}

\icmlcorrespondingauthor{Marvin Li}{marvinli@college.harvard.edu}
\icmlcorrespondingauthor{Sahil Kuchlous}{sahilkuchlous@gmail.com}
\icmlcorrespondingauthor{Jeffrey Wang}{jgwang@college.harvard.edu}

\icmlkeywords{Diffusion models, Algorithmic Fairness, Multicalibration}

\vskip 0.3in
]

\printAffiliationsAndNotice{\icmlEqualContribution} 

\begin{abstract} 
With the growing adoption of Text-to-Image (TTI) systems, the social biases of these models have come under increased scrutiny. Herein we conduct a systematic investigation of one such source of bias for diffusion models: embedding spaces. First, because traditional classifier-based fairness definitions require true labels not present in generative modeling, we propose statistical group fairness criteria based on a model's internal representation of the world. Using these definitions, we demonstrate theoretically and empirically that an unbiased text embedding space for input prompts is a \emph{necessary} condition for representationally balanced diffusion models, meaning the distribution of generated images satisfy diversity requirements with respect to protected attributes. Next, we investigate the impact of biased embeddings on evaluating the alignment between generated images and prompts, a process which is commonly used to assess diffusion models. We find that biased multimodal embeddings like CLIP can result in lower alignment scores for representationally balanced TTI models, thus rewarding unfair behavior. Finally, we develop a theoretical framework through which biases in alignment evaluation can be studied and propose bias mitigation methods. By specifically adapting the perspective of embedding spaces, we establish new fairness conditions for diffusion model development and evaluation. 
\end{abstract}

\section{Introduction}

As Text-To-Image (TTI) models become increasingly adept at generating complex and realistic images, they are being integrated into a wide range of commercial and creative services \cite{ramesh2021zeroshot, rombach2022highresolution, saharia2022photorealistic}. The proliferation of these models is evident in various industries; for example, advertising agencies use them to quickly generate visual content for campaigns, while film and game developers employ them to design detailed backgrounds and characters. The ubiquity of these models, however, has also raised significant ethical and fairness concerns, from their potential for misuse to the biases they encode. Addressing these issues is critical to developing TTI models that are not only technologically advanced but also socially responsible and inclusive. 

Here, we specifically study diffusion models, a class of models that form the centerpiece of most advanced TTI systems. We first consider \emph{direct representational harm}; that is, breaches of fairness that arise from generations that are imbalanced across different protected classes. In the literature, it is generally ``well-established" that most diffusion models have imbalanced representation in their generations. For instance, \citet{perera2023analyzing} show that popular TTI models consistently underrepresent minorities in most professions, a finding corroborated by several other papers \cite{wang2023t2iat, luccioni2023stable, wan2024survey}. 

Violations of fairness also arise indirectly; most prominently, a model that generates lower quality images of one group versus another creates \emph{indirect representational harm}. To audit this type of harm, one important evaluation criteria on which TTI models are often benchmarked is the alignment/faithfulness of images to prompts---how well the contents of the generated image match the prompt. If generated images are well-aligned with the prompt across classes, then the model treats these groups equally. While some papers attempt to use human ratings for benchmarking \cite{lee2023holistic}, this approach is expensive and hard to scale; thus, more recent work has explored automating alignment evaluation. One method growing in popularity is CLIPScore, based on the text-image embedding CLIP \cite{radford2021learning}, which measures the cosine similarity between a prompt and an image in the multimodal latent space \cite{hessel2022clipscore}. 

Embeddings play a critical role in generating from and evaluating diffusion models. Besides the dataset a model is trained on, the text embedding of the prompt is the only input to a diffusion model; similarly, many popular approaches for prompt-image alignment leverage a joint text-image latent space. Despite their critical role, however, few papers focus on connecting properties of embedding spaces to downstream fairness criteria. In this paper, we demonstrate that two intuitive conclusions hold both theoretically and empirically. In Section~\ref{sec:bias_embed_biased_gen}, we show that biased embeddings cause biased generations in diffusion models. In Section~\ref{sec:auditing}, we show biased image-text embedding spaces lead to biased evaluation of prompt-image alignment for \emph{any} TTI model. As modern machine learning systems increasingly integrate many components, like learned embeddings, our work highlights that the fairness of each part is critical for ensuring the fairness of the whole. 
\section{Related Work}

\textbf{Theory of Algorithmic Fairness.} Existing work in algorithm fairness has concentrated on the setting of supervised learning, where individuals are mapped to outcomes. Traditionally, group fairness notions (like statistical parity) enforce some measure of average equal treatment between members of protected classes, whereas individual fairness notions require similar individuals to be treated similarly, under some task-specific metric \cite{dwork2011fairness}. More recently, \emph{multi-group} fairness has emerged as a middle ground, which enforces fairness constraints on (up to exponentially many) subgroups within the dataset. One such notion is multicalibration \cite{pmlr-v80-hebert-johnson18a}, on which we base our work in Section~\ref{sec:auditing}, although others have also been proposed \cite{kearns2018preventing}. 

\textbf{Theory of Diffusion Models.} In brief, diffusion models sample from some distribution of images by learning how much noise gets added to images over time so images can be sampled from that distribution by \emph{de}noising a sample of pure Gaussian noise; the denoising function at time $t$ is the same as the statistical \emph{score} $\nabla \ln q_t(x)$ of the noised distribution $q_t(x)$ \cite{song2020score}. Mathematically, the convergence to the true distribution of the sampling process with a sufficiently good approximation of the score can be proved with Girsanov's thoerem (Theorem~\ref{thm:girsanov}); empirically, the process is implemented via a discrete-time approximation, where we iteratively denoise images over small time steps. We defer a full theoretical coverage of diffusion models to Appendix~\ref{app:diffuse-math}, and note that we later leverage Theorem~\ref{thm:girsanov} for our proof that biased embeddings cause biased generations.

\textbf{Bias in Embedding Spaces.} \citet{bolukbasi2016man} was the first to demonstrate bias present in word embeddings on the basis of gender, and proposed methods to debias such embeddings, followed by similar work on bias in race \cite{manzini-etal-2019-black, dev2019attenuating}. \citet{papakyriakopoulos2020bias} train a sentiment classifier on a biased word embedding and illustrate that the downstream classifier's outputs reflect the direction of biases in the input embedding. These biases have also been observed in multimodal embeddings like CLIP \cite{wang2021genderneutral, berg2022prompt}, and the impact of these biases on evaluating image captioning methods has been studied \cite{qiu2023genderbiasesautomaticevaluation}. However, to our knowledge, our work is the first to analyze the effect of bias in learned embedding spaces like CLIP on the evaluation of TTI systems, as well as the downstream impact of a biased embedding as a component of a \emph{generative} model. 

\textbf{Bias in Diffusion Models.} Several recent papers on fairness in diffusion models propose methods to sample with more equal representation across protected classes \cite{li2024fair, friedrich2023fair, choi2024fair, shen2024finetuning, chuang2023debiasing}. Of note, \citet{chuang2023debiasing} and \citet{li2024fair} intervene on the prompt embedding, using a debiasing projection and a learned fair representation of prompts, respectively. Both works illustrate that debiasing embeddings leads to more representationally fair generations. In Section~\ref{sec:bias_embed_biased_gen}, we show this is a necessary condition. 

\textbf{Prompt-Image Alignment.} There are several papers that attempt to benchmark the prompt-image alignment of image generation models.  Of these, \citet{lee2023holistic} rely most heavily on CLIPScore for alignment evaluation. \citet{bakr2023hrsbench} use CLIPScore for more specific evaluation criteria like generation of emotion. \citet{chen2024evaluating} define a composite metric called `text condition evaluation' that encompasses alignment and fairness. To calculate the alignment portion of their score, the authors utilize a visual question answering (VQA) model called BLIP \cite{li2022blip}. \citet{bakr2023hrsbench} and \citet{chen2024evaluating} also raise concerns with utilizing CLIPScore to measure alignment, but do not cite fairness as a concern. There have been several other attempts at evaluating text-image alignment \cite{hu2023tifa, xu2023imagereward, yarom2023read}, all of which potentially demonstrate some form of implicit bias due to their reliance on external models. In this paper, we introduce a framework for studying these biases irrespective of the underlying alignment score function used.

\textbf{Algorithmic Fairness for Generative Models.} In the traditional theory of algorithmic fairness for supervised classifiers, individuals self-identify into different categories and fair algorithms provide guarantees with respect to how individuals in different (sub)groups are treated \cite{disparateimpact}. With generative models, however, we run up against a fundamental epistemic question of identification: generations have no underlying classifications. While most extant literature on fairness in TTI models relies on a different classifier to categorize generations \cite{friedrich2023fair, shen2024finetuning}, this method suffers from two core flaws. First, it relies on physical features alone to define categories, which has the potential to reify stereotypes. Second, these classifiers still exhibit various social biases reflected in their training data \cite{luccioni2023stable}; in fact, Section~\ref{sec:auditing} demonstrates that the biases of such classifiers impairs model evaluation. To this end, we attempt to leverage the model's \emph{internal representation} of the world in order to construct notions of fairness when defining biased generations. In particular, we imagine that the model knows the class from which it is generating; while we don't have access to this latent truth, we can operationalize our definitions around it, as we do in Section~\ref{sec:bias_embed_biased_gen}. 

\section{Preliminaries}

Here, we introduce some notation that we will use throughout the manuscript. Let $\TV$ denote the total variation distance. Let $\mathcal{P}$ denote the set of all prompts (e.g. `an image of a doctor interacting with a patient') and $\mathcal{I}$ to denote the 
set of all images. A text-to-image model $M : \mathcal{P} \to \Delta(\mathcal{I})$ takes a prompt as the input and returns a distribution over images that can be sampled from. A multimodal embedding consists of a pair of functions $(e_{\cI} : \cI \to \mathcal{S}^{n-1}, e_{\cP} : \cP \to \mathcal{S}^{n-1})$, where $\mathcal{S}^{n-1}=\{x\in \R^n:\|x\|=1\}$. We use the set $\cA = \{a_1, ..., a_k\}$ to refer to possible attributes such as race or gender specifiers (e.g. `female') and $b \in \cP$ to denote the base prompt (e.g. `firefighter,' or more generally any descriptor independent of $\cA$). For some $a \in \cA$, the operation $a+b \in \cP$ represents the prompt obtained by combining an attribute and a base prompt. For example, for $a = \text{`male'}$ and $b = \text{`doctor'}$, we have $a + b = \text{`male doctor'}$. We also use $e_{\cP}$ to denote the map from prompts to prompt embeddings within a diffusion model.  Let $p_{y}:\mathcal{I} \to \mathbb{R}^{\geq 0}$ be the distribution over images generated by the model conditioned on text prompt $y$. The above notation applies to any text-to-image model; for diffusion models specifically, let $s_t(x,y)$ denote the learned score function of input $x$ and time $t$ conditioned on a prompt embedding $y$. Let $T$ be the length of the denoising period.

\section{Biased Embedding, Biased Generations}
\label{sec:bias_embed_biased_gen}
In Theorem \ref{thm:lipschitz}, we prove that sufficiently strong bias in the prompt embeddings implies representational imbalance in the image generations. We start with some intuition to interpret the result. Suppose we have a base prompt $b=\text{`nurse'}$ and attributes $a_1=\text{`man'}$ and $a_2=\text{`woman'}$. If the embedding of $b$ is sufficiently close to the embedding for $a_2+b$ and images with the attributes $a_1$ and $a_2$ are distinguished from each other, Theorem~\ref{thm:lipschitz} states that the diffusion model conditioned on $b$ will mostly produce images similar to when it is conditioned on $a_2+b$. Thus most of the images generated from $\text{`nurse'}$ would be of $\text{`woman nurse'}$, which is exactly representational bias in sampling. To formalize this, we first operationalize our notions of biased prompt embeddings and representational balance that will be relevant throughout this section. 

\textbf{Biased Embeddings.} Extant literature on biased embeddings identifies imbalanced distances between words (e.g. if `woman' is closer to `nurse' than `man'). We leverage a similar, albeit more general, notion here. 
\begin{definition}
Given set of attributes $\mathcal{A}$, base prompt $b$, and some $\varepsilon>0$, an embedding for $b$ is $\varepsilon$-close to $a_i \in \mathcal{A}$  if $\|e_{\cP}(a_i+b)-e_{\cP}(b))\| \leq \varepsilon$.    
\end{definition}
This is extremely similar to the definition of bias from \citet{bolukbasi2016man}, which computes the projection of some base prompt to a learned gender direction. Directly computing the distance between a base prompt and the base prompt with an associated attribute is a natural extension of this idea and also circumvents the assumption in \citet{bolukbasi2016man} that there is an explicit gender direction in the latent space. While closeness does not immediately imply a problematic embedding (e.g. if $\varepsilon$ is large), common biases in embeddings satisfy $\varepsilon$-closeness to some $a_i \in \mathcal{A}$ with small $\varepsilon$ (i.e. places the base prompt closer to some protected attribute than others). 

\textbf{Representational Balance.} Now that we have defined a biased embedding, we define our definition of fairness. Informally, a model is representationally balanced if the proportion of images generated of prompt $b$ with attribute $a_i$ is at least $v_i>0$ for $i \in [k]$. In this section, we rely on the model's own representations for different attributes $a_i$. 
\begin{definition}\label{def:rep_bal}
A model $p_y$ is representationally balanced with respect to attribute set $\mathcal{A}$, base prompt $b$, and constants $(v_i)_{i=1}^{k}$ if $\TV(p_b,p_{a_i+b}) \leq 1-v_i$ and $v_i > 0$ for $i \in [k]$. \end{definition}

This both ensures that $p_b$ produces at least $v_i$-proportion of images in the support of $p_{a_i+b}$ and that $p_b$ does not place too much mass on images that are unlikely to be generated by $a_i+b$. To gain intuition into this definition, we can view $p_b$ as a mixture model consisting of images with and without $a_i$, with mixture weights $v_i$ and $1-v_i$, respectively. If the components are disjoint and $p_{a_i+b}$ equals the component of images with $a_i$, then we have exactly $\TV(p_b,p_{a_i+b})=1-v_i$. We note that this definition does not control for small discrepancies between sampling directly from $p_{a_i+b}$ versus sampling from $p_b$ and conditioning on belonging to the support of $p_{a_i+b}$. For example, let $p_b$ place mass $0.5,0.1,0.4$ on three images, and $p_{a_i+b}$ places mass of $0,0.5,0.5$ on the same images. Clearly the $\TV$ is still $0.5$, and this is representationally fair for $v_i=0.5$, but the $a_i+b$-images generated by $p_b$ are different than those generated by $p_{a_i+b}$. This example arises because our Definition~\ref{def:rep_bal} is a \emph{group} fairness notion. For finer control over representation of specific images, one can concatenate subsets of attributes together for \emph{multi-group} representational control; we view this line of work as an interesting future direction. 

We now state the two assumptions of Theorem~\ref{thm:lipschitz}. The most important one is that the score function is Lipschitz with respect to the prompt embeddings (Assumption~\ref{assump:lipschitz}). Lipschitzness of the learned score function in the first argument is standard in the theory of diffusion literature, e.g., \citet{DBLP:conf/iclr/ChenC0LSZ23}. This assumption can also be interpreted as an \emph{individual fairness constraint}, where the requirement that similar prompt embeddings produces similar images parallels the notion that similar individuals are mapped to similar outcomes \citep{dwork2011fairness}. Assumption~\ref{assum:distinct} codifies the requirement that $\mathcal{A}$ represent distinct categories, which is necessary to show violations of our definition of representational balance because it implies that over-representation of one group limits representation of other groups. 

\begin{assumption}[Lipschitz in prompt embeddings]\label{assump:lipschitz}
The score $s_t(x,y)$ of the diffusion model is $L$-Lipschitz in $y$. 
\end{assumption}
\begin{assumption}[Distinct categories of identifiers]\label{assum:distinct}
There exists $0 < \varepsilon < \frac{\min_i v_i}{2}$ such that for all distinct $a_i,a_j \in \cA$, $\TV(p_{a_i+b},p_{a_j+b})\geq 1-\varepsilon$. 
\end{assumption}
Now we formally state Theorem~\ref{thm:lipschitz}. 
\begin{theorem}[Bias in embeddings implies bias in image generations]
\label{thm:lipschitz}
Assume WLOG we have an embedding for $b$ that is $\frac{\varepsilon}{\sqrt{T}L}$-close to $a_1$, i.e. $\|e_{\cP}(b)-e_{\cP}(a_1+b)\| \leq \frac{\varepsilon}{\sqrt{T}L}.$ Under Assumptions~\ref{assump:lipschitz} and~\ref{assum:distinct}, we have 
 $\TV(p_{b},p_{a_1+b}) \leq \varepsilon$ and $\TV(p_{b},p_{a_j+b}) \geq 1-2\varepsilon > 1-\min_i v_i$ for $j \ne 1$. 
\end{theorem}

\begin{proof}
We first show that $\TV(p_{b},p_{a_1+b}) \leq \varepsilon$. We will do this by writing a bound on $\KL(p_{b},p_{a_1+b})$, and translating it to a bound on total variation via Pinsker's Inequality. First, we leverage Theorem~\ref{thm:girsanov}, which bounds the KL divergence $\KL(p_{b},p_{a_1+b})$ by the quantity
$$\int_0^T \E_{x_t} \|s_t(x_t,e_{\cP}(b))-s_t(x_t,e_{\cP}(a_1+b))\|^2 dt.$$
To evaluate this, we note that by Assumption~\ref{assump:lipschitz}, we have an upper bound $\|s_t(x_t,e_{\cP}(b))-s_t(x_t,e_{\cP}(a_1+b))\| \leq L\|e_{\cP}(a_1+b)-e_{\cP}(b)\|$ for all $x_t$. Hence, 
$$\KL(p_{b},p_{a_1+b}) \leq T(L\|e_{\cP}(a_1+b)-e_{\cP}(b)\|)^2$$
Finally, applying Pinsker's Inequality, we have that 
\begin{align*}
    \TV(p_{b},p_{a_1+b}) &\leq \sqrt{\frac{1}{2} \KL(p_{b},p_{a_1+b})} \\
    &\leq \sqrt{TL^2 \|e_{\cP}(a_1+b)-e_{\cP}(b)\|^2} \\ &\leq \varepsilon
\end{align*}
where the last inequality comes via our $\frac{\varepsilon}{\sqrt{T}L}$-close assumption. 

Next, we exhibit $\TV(p_{b},p_{a_j+b}) \geq 1-2\varepsilon > 1-\min_i v_i$ for $j \ne 1$. By the reverse triangle inequality,
$$\TV(p_{b},p_{a_j+b}) \geq \TV(p_{a_1+b},p_{a_j+b})-\TV(p_{b},p_{a_1+b}).$$
By Assumption~\ref{assum:distinct}, we have that $ \TV(p_{a_1+b},p_{a_j+b}) \geq 1-\varepsilon$ for $j \ne 1$, so in total we have that
\begin{align*}
    \TV(p_{b},p_{a_j+b}) &\geq \TV(p_{a_1+b},p_{a_j+b})-\TV(p_{b},p_{a_1+b}) \\
    &\geq 1-\varepsilon-\varepsilon \\
    &=1-2\varepsilon \\
    &> 1 - \min_i v_i \quad \text{(Assumption~\ref{assum:distinct})}.
\end{align*}
\end{proof}

This theorem implies that a base prompt $b$ sufficiently close to $a_1+b$ in the embedding will have total variation bounded by $\varepsilon$ between the resulting distributions; similarly, because the attributes are distinct, the total variation between $b$ and $a_j+b$ for $j \neq 1$ is lower-bounded by $1-2\varepsilon$. Hence, $a_1+b$ is generated more, breaking representational balance. Note that an $\frac{\varepsilon}{\sqrt{T}L}$-close prompting embedding is an extremely strong requirement yet necessary to meaningfully control the total variation. As such, we next turn to an empirical investigation of the relationship between bias in prompt embeddings and representational imbalance to verify our theoretical conclusions.  

\textbf{Empirics.} Previous work has illustrated that the output of diffusion models is not representationally balanced. For instance, \citet{friedrich2023fair} illustrate gender bias in the outputs of occupations queried of Stable Diffusion which reflect both the direction of biases in CLIP embeddings and representational imbalances in the underlying training dataset. We find a similar relationship when comparing generations across occupations from Stable Diffusion 2.1 (SD2.1) against biases in the underlying CLIP embedding. 

Since SD2.1 is trained on an underlying dataset (LAION-5B) that is itself representationally imbalanced, however, it is unclear if the bias in image outputs is caused by a biased prompt embedding or imbalanced training dataset. To probe this, we train a conditional diffusion model from scratch with balanced training data across three classes (nurse, philosopher, and person) but with biased prompt embeddings (where nurse is closer to woman, philosopher is closer to man, and person is roughly equidistant). We find that the resulting model \emph{is} biased in its generations in the same direction as the embedding, with women as the majority of nurse generations, men as the majority of philosopher generations, and a roughly even split in generations of people. We defer the full experimental details to Appendix~\ref{app:generations}. 

\section{Bias in Alignment Auditing}
\label{sec:auditing}

In this section we consider the problem of auditing image generation models for prompt-image alignment and analyze the impact that biased multimodal embeddings may have on the fairness of alignment scores. We note that from here on out, we refer to ``score" as the alignment score, not the learned function of a diffusion model. While the results in \Cref{sec:bias_embed_biased_gen} studies the representational harm in the generations of biased diffusion models, this section analyzes harms caused by biased alignment predictors and therefore borrows from richer notions of fairness proposed for classification systems. We define a notion of fairness for alignment functions and demonstrate \emph{necessary} conditions for multimodal embeddings to satisfy this definition. Additionally, we evaluate the bias of existing auditing functions and suggest simple techniques for mitigating such biases. 

\subsection{Definitions of Fairness}\label{sec:fairness}

We begin by defining what it means for an alignment auditing function to be fair. These definitions are inspired by existing work on fairness for predictors introduced by \cite{pmlr-v80-hebert-johnson18a}, and are designed to capture the idea that the alignment of an image with a prompt should be independent of protected attributes like race and gender when they are not explicitly specified in the prompt. Note that this differs from our definition of representational balance in Section~\ref{sec:bias_embed_biased_gen}: the auditing function outputs a scalar, allowing us to borrow definitions from the rich extant fairness literature. We use $s^* : \mathcal{P} \times \mathcal{I} \to [0, 1]$ to denote the true alignment score and $s : \mathcal{P} \times \mathcal{I} \to [0, 1]$ to denote an auditing function.

\begin{definition}[Multiaccuracy]
    Let $\cC \subseteq 2^\cI$ be a collection of subsets of $\cI$ and $\alpha \in [0, 1]$. An auditing function $s$ is $(\cC, \alpha)$-multiaccurate for prompt $b \in \cP$ if, for all $I \in \cC$,
    $$\left|\E_{i \sim I} [s^*(b, i)-s(b, i)]\right| \leq \alpha.$$
\end{definition}

In other words, multiaccuracy for a prompt $b$ ensures that the function $s$ is $\alpha$-accurate on every subset of images $I \in \cC$. To see how this notion is useful for fairness, consider a set of protected attributes $a_1, ..., a_k$ for a base prompt $b$, and let $I_\ell$ be a subset of images corresponding to prompt $b$ with attribute $a_\ell$. If $\cC = \{I_1, ..., I_k\}$, $(\cC, \alpha)$-multiaccuracy guarantees that the auditing function $s$ is $\alpha$-accurate on average on each of the protected attributes. Thus, for example, if the prompt we care about is ``doctor" and the attributes we care about encompass gender, setting $I_1$ and $I_2$ to images of male and female doctors respectively gives us the guarantee that $s$ is $\alpha$-accurate on both genders.\footnote{Note that gender is not a binary, and we would hope that generative models are able to capture gender on a spectrum. We use male and female in the example for ease of notation.} We analyze the strengths and weaknesses of multiaccuracy and define a stronger fairness guarantee based on multicalibration in Appendix~\ref{app:multicalibration}.

To see why these notions of fairness are useful for auditing text-to-image models, we prove the following theorem. Intuitively, this theorem states that if we have subsets of images with particular attributes such that the true alignment is equal in expectation across attributes, and if the auditor is multiaccurate on these subsets of images, then the alignment score remains stable irrespective of how a text-to-image model chooses to sample from these attributes.

\begin{theorem}
    Consider a base prompt $b \in \cP$ and attributes $\cA = \{a_1, ..., a_k\}$, and let $I_\ell$ be a subset of images corresponding to prompt $b$ with attribute $a_\ell$. Assume that $\E_{i \sim I_\ell}[s^*(b, i)] = \bar{s}$ for all $\ell \in [k]$. Let $\cC = \{I_1, ..., I_k\}$. Consider a model $M$ that, given the prompt $b$, returns an image $i \sim I_\ell$ with probability $p_\ell$, where $\sum_{\ell \in [k]} p_\ell = 1$. If the auditing function $s$ is $(\cC, \alpha)$-multiaccurate, then $\left|\E_{i \sim M(b)}[s(b, i)]-\bar{s}\right| \leq \alpha$ irrespective of the probabilities $p_1, ..., p_k$. 
\end{theorem}

\begin{proof}
    We see that
    \begin{align*}
        \E_{i \sim M(b)}[s(b, i)] &= \sum_{\ell \in [k]} \E_{i \sim I_\ell}[s(b, i)] \cdot p_\ell\\
        &\leq \sum_{\ell \in [k]} \left(\E_{i \sim I_\ell}[s^*(b, i)] + \alpha\right) \cdot p_\ell\\
        &= \sum_{\ell \in [k]} (\bar{s}+\alpha) \cdot p_\ell\\
        &= \bar{s}+\alpha.
    \end{align*}
    By the same logic, we also see that $\E_{i \sim M(b)}[s(b, i)] \geq \bar{s}-\alpha$. Thus, $-\alpha \leq \E_{i \sim M(b)}[s(b, i)]-s \leq \alpha$, so $\left|\E_{i \sim M(b)}[s(b, i)]-\bar{s}\right| \leq \alpha$. 
\end{proof}

As a corollary of this theorem, we see that irrespective of the probabilities $p_1, ..., p_n$, the difference in the alignment score of two models is at most $2\alpha$. Thus, going back to our example of male and female doctors, we see that as long as we expect the average alignment of an image of a male doctor and a female doctor with the prompt ``doctor" to be the same, if our auditing function is multiaccurate on the appropriate subsets of images, models will get similar alignment scores irrespective of what proportion of male or female doctors they generate.

Finally, it is worth noting that checking whether an auditing function is multiaccurate may be challenging, since this requires access to the true alignment scores $s^*$. Note that if we had arbitrary oracle access to $s^*$ there would be no need to evaluate auditing functions $s$ at all, since we could just use $s^*$ to audit the alignment of text-to-image models. However, the advantage of our notion of fairness is that it only requires ``true" scores for a diverse but fixed set of images on which auditing functions can then be evaluated. Such datasets could be obtained by manual labeling, and there have been several efforts to create large image-caption datasets with alignment ratings \cite{levinboim-etal-2021-quality, lee-etal-2021-umic, vedantam2015cider, hodosh2013framing}.

\subsection{Properties of Fair Embeddings} \label{sec:fairembeddings}

In this section, we explore necessary conditions for multi-modal embeddings to satisfy multiaccuracy. Let us begin by defining an auditing function in terms of a multi-modal embedding space. For vectors $\bx$ and $\by$, we define their cosine similarity as
$\cos(\bx, \by) = \frac{\bx \cdot \by}{\|\bx\|\|\by\|}.$
Given a multimodal embedding $(e_\cI, e_\cP)$, we define the auditing function
$s(b, i) = \frac{\cos\bigl(e_\cP(b), e_\cI(i)\bigr)+1}{2}.$ This ensures that $s(b, i) \in [0, 1]$. Note that existing techniques like CLIPScore \cite{hessel2022clipscore} achieve a similar effect by clipping the cosine similarity at a minimum of 0, but this is equivalent to our definition up to a factor of 2 as long as similarity is positive. 

Given the definitions of fairness for alignment auditors in \Cref{sec:fairness}, it is natural to ask what properties we would expect from a multimodal embedding space in order to obtain a fair auditor. We start with the following observation, which provides a necessary condition for multiaccuracy based on the average score between a prompt and different subsets of images. 

\begin{restatable}{theorem}{textimage}\label{thm:text-image}
    If $s$ is $(\cC, \alpha)$-multiaccurate for a prompt $b$, for all $I_\ell, I_{\ell'} \in \cC$, if $\E_{i \in I_\ell} [s^*(b, i)] = \E_{i \in I_{\ell'}} [s^*(b, i)]$, then $|\E_{i \in I_\ell} [\cos(e_\cP(b), e_\cI(i))] - \E_{i \in I_{\ell'}} [\cos(e_\cP(b), e_\cI(i))]| \leq 4\alpha$. 
\end{restatable}

\begin{proof}
    Since $s$ is $(\cC, \alpha)$-multiaccurate for prompt $b$, for all $I \in \cC$ we know that 
    $$\left|\E_{i \sim I} [s^*(b, i)-s(b, i)]\right| = \left|\E_{i \sim I} [s^*(b, i)]-\E_{i \sim I}[s(b, i)]\right| \leq \alpha.$$
    Moreover, for some $I_{\ell}, I_{\ell'} \in \cC$, if $\E_{i \in I_{\ell}} [s^*(b, i)] = \E_{i \in I_{\ell'}} [s^*(b, i)]$, by the triangle inequality we see that 
    $$\left|\E_{i \sim I_\ell} [s(b, i)]-\E_{i \sim I_{\ell'}}[s(b, i)]\right| \leq 2\alpha.$$
    Substituting the definition of $s(b, i)$, we see that 
    $$\left|\E_{i \sim I_{\ell}} [\cos(e_\cP(b), e_\cI(i))]-\E_{i \sim I_{\ell'}}[\cos(e_\cP(b), e_\cI(i))]\right| \leq 4\alpha.$$
\end{proof}

This implies that given a prompt $b$ and subsets of images with similar average true alignment scores, if the embeddings of those images do not have a similar distance (on average) to the embedding of $b$ then our auditor is not multiaccurate on those images. 

Next, we use this observation to detect bias based only on the prompt embeddings $e_\cP$. As discussed in \Cref{sec:bias_embed_biased_gen}, prompt embeddings are often biased in themselves, which causes bias in diffusion models when these embeddings are used as inputs. By showing that fair prompt embeddings are necessary for multiaccuracy, the theorem below demonstrates that bias in prompt embeddings also translates to bias in multimodal embeddings, and therefore results in biased alignment auditing.

\begin{restatable}{theorem}{texttext}
    Consider a prompt $b$ and attributes $\cA = \{a_1, ..., a_k\}$. Let $I_\ell$ be a set of images such that for every $i \in I_\ell$, $e_\cI(i)$ is in a ball of radius $\varepsilon$ around $e_\cP(a_\ell + b)$. For $\cC = \{I_1, ..., I_k\}$, if $s$ is $(\cC, \alpha)$-multiaccurate for prompt $b$, for any $\ell, \ell' \in [k]$ such that $\E_{i \in I_\ell} [s^*(b, i)] = \E_{i \in I_{\ell'}} [s^*(b, i)]$, it holds that $|\cos(e_\cP(b), e_\cP(a_\ell + b)) - \cos(e_\cP(b), e_\cP(a_{\ell'} + b))| \leq 4\alpha+2\varepsilon$.  
\end{restatable}

\begin{proof}
    Since $s$ is $(\cC, \alpha)$-multiaccurate for prompt $b$, if $\E_{i \in I_\ell} [s^*(b, i)] = \E_{i \in I_{\ell'}} [s^*(b, i)]$, we know by \Cref{thm:text-image} that 
    $$\left|\E_{i \sim I_{\ell}} [\cos(e_\cP(b), e_\cI(i))]-\E_{i \sim I_{\ell'}}[\cos(e_\cP(b), e_\cI(i))]\right| \leq 4\alpha.$$

    Next, consider the expression $\cos(e_\cP(b), e_\cI(i))$. For all $i \in I_\ell$, since $e_\cI(i)$ is in a ball of radius $\varepsilon$ around $e_\cP(a_\ell + b)$, we know that there exists a vector $\delta$ such that $e_\cI(i) = e_\cP(a_\ell + b) + \delta $ and $\|\delta\| \leq \varepsilon$. Thus, we see that
    \begin{align*}
        \cos(e_\cP(b), e_\cI(i)) &= \cos(e_\cP(b), e_\cP(a_\ell + b) + \delta)\\
        &= \frac{e_\cP(b) \cdot (e_\cP(a_\ell + b) + \delta)} {\|e_\cP(b)\|\|e_\cP(a_\ell + b) + \delta\|}\\
        &= e_\cP(b) \cdot e_\cP(a_\ell + b) + e_\cP(b) \cdot \delta\\
        &= \cos(e_\cP(b), e_\cP(a_\ell + b)) + e_\cP(b) \cdot \delta.
    \end{align*}
    Since $\|\delta\| \leq \varepsilon$, we know that $|e_\cP(b) \cdot \delta| \leq \varepsilon$. Thus, 
    $$|\cos(e_\cP(b), e_\cI(i)) - \cos(e_\cP(b), e_\cP(a_\ell + b))| \leq \varepsilon,$$
    for all $i \in I_\ell$, which implies that
    $$\left|\E_{i \sim I_{\ell}} [\cos(e_\cP(b), e_\cI(i))] - \cos(e_\cP(b), e_\cP(a_\ell + b))\right| \leq \varepsilon.$$
    By symmetry, note that the same expression holds for $i \sim I_{\ell'}$ and $a_{\ell'}$. Thus, by the triangle inequality, 
    $$|\cos(e_\cP(b), e_\cP(a_\ell + b))-\cos(e_\cP(b), e_\cP(a_{\ell'} + b))| \leq 4 \alpha + 2 \varepsilon.$$
\end{proof}

The contrapositive of this theorem implies that, for a prompt $b$, if there are two attributes $a_1$ and $a_2$ such that the embeddings of $a_1+b$ and $a_2+ b$ are not equidistant from the embedding of $b$, then a multiaccurate auditing function $s$ derived from this embedding can not be fair on images that closely match attributes $a_1$ and $a_2$. 

Thus, if the underlying embedding space is not fair on prompts, we should not expect it to be a fair auditor. This allows us to infer bias in a multimodal embedding space by only looking at its prompt embedding function $e_\cP$.

\subsection{Auditing Alignment with Biased Embeddings}\label{sec:biased-audit}

While there have been some efforts to debias CLIP and CLIPScore \cite{dehdashtian2024fairerclip}, it is likely that many of the approaches we use to audit prompt-image alignment will continue to exhibit some form of representational bias. Thus, it is important to consider methods to mitigate this bias in the score calculation phase. Currently, the standard method for calculating the alignment score of a model $M$ for a prompt $b$ using a multimodal embedding is to take the average of the auditing function described in \Cref{sec:fairembeddings} over several images samples from $M(b)$ \cite{lee2023holistic}. Henceforth, we will refer to this method as \textit{score-then-average}. As we have discussed, one of the ways in which an embedding space may be biased is if for a prompt $b$, vectors corresponding to images with attribute $a_1$ are more similar to the vector for $b$ on average than those with attribute $a_2$. In this case, a model that generates only images with attribute $a_1$ will consistently get a higher alignment score than a model that generates a balanced distribution of images. In this section, we suggest an alternative method of evaluating alignment based on potentially biased multimodal embeddings that helps alleviate this form of bias. Intuitively, this method measures the maximum similarity score between an image and the base prompt with different protected attributes attached. We call this method \textit{subclass-score}. 

In this approach, given a prompt $b$, we begin with a list of attributes $a_1, ..., a_k$ on which we would like to be fair. To calculate the alignment score of a model $M$ of this prompt, we start by sampling a set $I$ of images from $M(b)$. We then calculate the individual scores of images $i \in I$ as
$$s(b, i) = \max_{\ell \in [k]}(s(a_k + b, i)).$$ For example, if $a_1 = \text{`male'}$ and $b = \text{`doctor'}$,  $a_1 + b = \text{`male doctor'}$. Finally, to calculate the aggregate score we take the average of $s(b, i)$ over images $i \in I$. 

This aims to tackle the issue where embeddings of images with one attribute may be closer on average to the embedding of the prompt than images with other attributes. By measuring the similarity with attribute-specific vectors instead, the relative distance from the embedding of the original prompt is no longer an issue. We note that there remain drawbacks with this approach. First, it may be difficult to identify what attributes are appropriate for which prompts. For instance, Google's Gemini model recently attempted to impose diversity into prompts where it was not historically accurate or appropriate \cite{gautam2024melting}. Also, this does not guarantee fairness conditions like multiaccuracy. For example, it is possible that images with some attribute $a_\ell$ are closer on average to the prompt $a_\ell + b$ than another attribute $a_{\ell'}$ are to the prompt $a_{\ell'} + b$.

We introduce a second method for generating unbiased alignment scores, \textit{average-then-score}, in Appendix~\ref{app:avgscore}. This method involves taking the mean of the image embeddings before computing the cosine similarity to the prompt embedding. We evaluate and compare these approaches in \Cref{sec:clip}.

\subsection{Empirical Case-Study: CLIP}\label{sec:clip}

In this section, we investigate biases present in an existing method of alignment auditing called CLIPScore \cite{hessel2022clipscore} based on the multimodal embedding model CLIP \cite{radford2021learning}. We then evaluate the techniques proposed in \Cref{sec:biased-audit}. We find that the text embeddings of CLIP demonstrate gender bias in the representation of occupations, matching the findings of \cite{bolukbasi2016man}. We also find that a similar bias is present in the relationship between text embeddings of occupations and gendered image embeddings. In particular, we see that for the same set of images, male medical professionals got a 0.02 higher score than female medical professionals on average for the prompt ``doctor," and a 0.061 lower score for the prompt ``nurse." This clearly demonstrates an underlying bias in the auditing CLIPScore function; models that generate only male doctors and only female nurses get over $5\%$ higher alignment scores than models with balanced distributions. Finally, we evaluate our bias mitigation methods in \Cref{sec:biased-audit} and find that subclass-score performs significantly better than average-then-score, staying roughly consistent across gender ratio. Thus, subclass-score is the a promising step for alleviating gender bias in alignment scores for text-to-image models. Further details can be found in Appendix~\ref{sec:mitigating-bias}.

\section{Conclusion and Future Work}
In this paper we study the relationship between bias in embeddings and image generations for text-to-image models. We propose new statistical fairness criteria for generative models, and we prove theoretically and empirically that biased prompt embeddings lead to representationally unfair outputs. This establishes a \emph{necessary} condition (an unbiased embedding) for an unbiased model. We then investigate the impact of biased prompt embeddings on measuring the alignment between image generations and prompts, as well as multi-accuracy and multi-calibration style definitions for fair alignment auditing algorithms.

There are several interesting directions to take this work. First, the text embedding's impact on fairness in diffusion models can be further investigated empirically with recently-released open state-of-the-art models that have public training sets \cite{gokaslan2023commoncanvas}. Next, we are interested in potential augmentations for training or evaluating diffusion models that would compensate for bias in embedding spaces. Finally, while we give several conditions that imply a metric-based auditing function is not fair, it would be interesting to to explore \emph{sufficient} conditions for fairness, like in \cite{dwork2011fairness}. 

Throughout our manuscript, we have worked around the question of true labels (e.g. race/gender) of artificial generations. In the traditional algorithmic fairness setting of supervised learning, we are satisfied with group algorithmic fairness guarantees because the underlying subjects self-identify into specific groups. Generated images have no such ground truth, which is why we proposed group fairness criteria in Section~\ref{sec:bias_embed_biased_gen} that bypass these labels. \citet{luccioni2023stable} proposes using Visual Question Answering model (VQA) for this purpose, but this introduces another source of bias with the VQA model. We hope that future literature provides a sound resolution to this epistemic question underlying any fairness auditing algorithm for text-to-image models. 

\section*{Acknowledgements}
The authors would like to thank Cynthia Dwork for teaching a class which inspired this project, many thoughtful discussions on fairness in generative models, and her abundant questions/suggestions which have greatly improved the clarity of this manuscript. Additionally, the authors would like to thank the ICML Workshop on Trustworthy Multimodal Foundation Models and Agents (TiFA) for featuring an initial version of this work, and comments from its poster/lightning talk. ML would also like to thank Sitan Chen for introducing him to the theory of diffusion models.

\section*{Social Impact Statement}

As diffusion models become more broadly used in diverse contexts and applications, it is important that their outputs follow various fairness desiderata. This paper studies several of these desiderata, and provides new methodologies and theoretical frameworks for auditing fairness in diffusion models. We also conduct extensive experiments on real diffusion models. We look forward to future research that will further investigate the questions we raise and consider and provide inclusive solutions.

\bibliography{refs}
\bibliographystyle{icml2024}

\newpage
\appendix
\onecolumn
\section{Technical Background on Diffusion Models}
\label{app:diffuse-math}

In this section, we review the basics of diffusion models and provide background for our theoretical results and experiments. 

At a high level, diffusion models seek to generate samples from a distribution given samples from that distribution; e.g. they learn to produce novel images of firefighters given a dataset of images of firefighters. They consist of a forward ``noising" process, which iteratively transforms the original distribution by adding small amounts of Gaussian noise to the sample, and a reverse process, which gradually transforms the pure noise distribution into the original distribution through a ``denoising" procedure. Intuitively, if one can (e.g. with a neural network) ``learn" the noise that is added to images in the initial phase, one should be able to recover an image by starting with pure Gaussian noise and subtracting the learned noise over time \cite{ho2020denoising}. In our experiments in this paper, we make a standard choice of denoising the images for $T=1000$ steps. The denoising function at time $t$ is closely related to the statistical \emph{score} $\nabla \ln q_t(x)$ of the noised distribution $q_t(x)$ \cite{song2020score}. Mathematically, the convergence to the true distribution of the sampling process with a sufficiently good approximation of the score can be proved with Girsanov's thoerem (see Section \ref{app:diffuse-math} in the Appendix); empirically, the process is implemented via a discrete-time approximation, where we iteratively denoise images over small time steps. In practice, diffusion models consist of the learned denoising function as well as an image encoder and decoder. The image decoder and encoder maps images to a latent space over which the diffusion process is applied. We also often want to guide the diffusion model towards a particular output; to do this, the diffusion model can be \emph{conditioned} to produce images consistent with a text prompt by allowing the denoiser to depend on a text embedding vector.

Here, we build up to Theorem \ref{thm:girsanov}, which we use as a black box in proving \ref{thm:lipschitz}. Consider a distribution $q$ over $\R^d$ with a smooth density function. Let $\mathcal{N}(0,\mathrm{Id})$ represent the multivariate normal distribution in $\R^d$. Different forward proceses of diffusion models are equivalent to the Ornstein-Uhlenbeck (OU) process up to a reparameterizing of time. The OU process takes samples from $q$ and transforms them into $\mathcal{N}(0,\mathrm{Id})$ by the following stochastic differential equation (SDE), 
\begin{equation}
    \D\forward_t = -\forward_t\,\D t + \sqrt{2}\D B_t\,, \qquad \forward_0 \sim q\,, \label{eq:forward}
\end{equation}
where the original sample $\forward_0$ is taken from $q$, $(\forward_t)_{t\ge 0}$ are the time-indexed random variables produced by the SDE after $t$ time, and $(B_t)_{t\ge 0}$ is the standard Brownian motion. We also define $q_t$ to be the distribution of $\forward_t$. Intuitively, the OU process dilutes the signal from the original sample $\forward_0$ with the $-\forward_t\,\D t$ term and gradually replaces it with the randomness in $\sqrt{2}\D B_t$. By the forward convergence of the OU process, the $\KL$ divergence between $q_t$ and $\mathcal{N}(0,\mathrm{Id})$ decays to $0$ exponentially quickly as $t \to \infty$. Now we consider the reverse process of the above SDE, in the sense that the marginal distributions $q_t$ are the same for both SDEs up to smooth test functions. We first fix a terminal time $T>0$ and consider the following SDE on the time interval $[0,T]$,
\begin{equation}
    \D\reverse_t = \{\reverse_t + 2\nabla\ln q_{T-t}(\reverse_t)\}\,\D t + \sqrt{2}\,\D B_t\,, \reverse_0 \sim q_T. \label{eq:reverse}
\end{equation}
where we start from samples $\reverse_0 \sim q_T$ drawn from the noised $q_T$ and $\nabla\ln q_{t}$ is the score function of the distribution $q_{t}$ at time $t$. We again take $(B_t)_{t\ge 0}$ to be the standard Brownian motion. It is well known that $q_{T-t}$ is the distribution of $\reverse_t$, which provides us with a general recipe for sampling from $q$: we start with samples $\reverse_0 \sim q_T$ and then apply the SDE in Eq.\ref{eq:reverse} to generate a sample from $q$. In practice, one does not have direct access to the score function $\nabla\ln q_{t}$ of the data distribution and must instead estimate it from training data. This can be estimated by appealing to Tweedie's formula, which relates the score to a denoising problem. We take samples $x \sim q$ and add noise $\eta \sim \mathcal{N}(0,\mathrm{Id})$ to form $e^{-t}x + \sqrt{1-e^{2t}}\eta$. We train a neural network $\mathrm{NN}_\theta(\cdot,t)$ to estimate $x$ from $e^{-t}x + \sqrt{1-e^{2t}}\eta$. By Tweedie's formula, the ``denoising" function which minimizes the mean squared error of $\mathbb{E}_{x \sim q, \eta\sim \mathcal{N}(0,\mathrm{id})}[\|x-\mathrm{NN}_\theta(\tilde{x},t)\|^2|\tilde{x}=e^{-t}x + \sqrt{1-e^{2t}}\eta]$ can be rearranged to yield the score up to known linear factors. 
\begin{lemma}[Tweedie's formula]
Given $\tilde{x}\sim x +e$ for $x \sim p$ and $e \sim \mathcal{N}(0,\sigma^2\mathrm{Id})$, the expectation of $x|\tilde{x}$ is 
\begin{align*}
\mathbb{E}[x|\tilde{x}]=\tilde{x}+\sigma^2 \nabla \ln \tilde{p}(\tilde{x}),
\end{align*}
where $\tilde{p}$ is the distribution of $\tilde{x}$. 
\end{lemma}
Until \cite{DBLP:conf/iclr/ChenC0LSZ23}, it was unclear whether this L2-approximation of the true score learned by the neural network could be used to faithfully sample from the base distribution. The below theorem, a simple consequence of Girsanov's theorem, gave an affirmative proof relating distribution learning to adequately solving the denoising problem up to average L2 error. 
\begin{theorem}[Section 5.2 of~\cite{DBLP:conf/iclr/ChenC0LSZ23}]\label{thm:girsanov}
    Let $(Y_t)_{t\in[0,T]}$ and $(Y'_t)_{t\in[0,T]}$ denote the solutions to
    \begin{align*}
        \D Y_t &= b_t(Y_t) \, \D t + \sqrt{2}\D B_t\,, \qquad Y_0 \sim q \\
        \D Y'_t &= b'_t(Y'_t)\, \D t + \sqrt{2}\D B_t\,, \qquad Y'_0 \sim q\,.
    \end{align*}
    Let $q$ and $q'$ denote the laws of $Y_T$ and $Y'_T$ respectively. If $b_t, b'_t$ satisfy that $\int^T_0 \mathbb{E}\,\norm{b_t(Y_t) - b'_t(Y_t)}^2\, \D t < \infty$, then $\KL(q\|q') \le \int^T_0 \mathbb{E}\,\norm{b_t(Y_t) - b'_t(Y_t)}^2\,\D t$.
\end{theorem}
This manuscript will rely on Theorem~\ref{thm:girsanov} as a black-box for future proofs. There are also several more discrepancies between this conceptual model of diffusion models and their practical implementation. First, $q_T$ is unknown and in practice one replaces it with the normal Gaussian. Second, one cannot perfectly simulate the SDE in Eq.\ref{eq:reverse} and must rely on numerical approximations. The errors incurred by both differences can be subsumed into the final distributional error bound with polynomial scaling with respect to the relevant parameters~\cite{DBLP:conf/iclr/ChenC0LSZ23}. 

\section{Details from Section~\ref{sec:bias_embed_biased_gen}}

\label{app:correlation}
\subsection{Biased Embeddings Correlate with Representationally Imbalanced Generations}

To provide empirical support for Theorem~\ref{thm:lipschitz}, we study gender bias in occupations for CLIP embeddings used as inputs for SD2.1 and images generated by SD2.1. We consider the $\texttt{Professions}$ dataset containing portraits of people in 146 different professions from SD2.1 \cite{luccioni2023stable}. For each profession, we compute the ratio of cosine similarity between the CLIP embeddings of the prompts ``\texttt{Profession, man}" and ``\texttt{Profession}," and the cosine similarity of the CLIP embeddings of the prompts ``\texttt{Profession, woman}" and ``\texttt{Profession}." If this ratio is above 1, then this profession is biased in embedding space towards men over woman. To measure the bias in SD2.1 outputs, we sort the image generations into men or women categories with CLIP-ViT-B/32. In Figure~\ref{fig:clip_distribution}, there is a mild correlation between gender bias in CLIP embeddings and gender bias in image generations for a given occupation.

\begin{figure}[H]
    \centering
\includegraphics[width=0.5\linewidth]{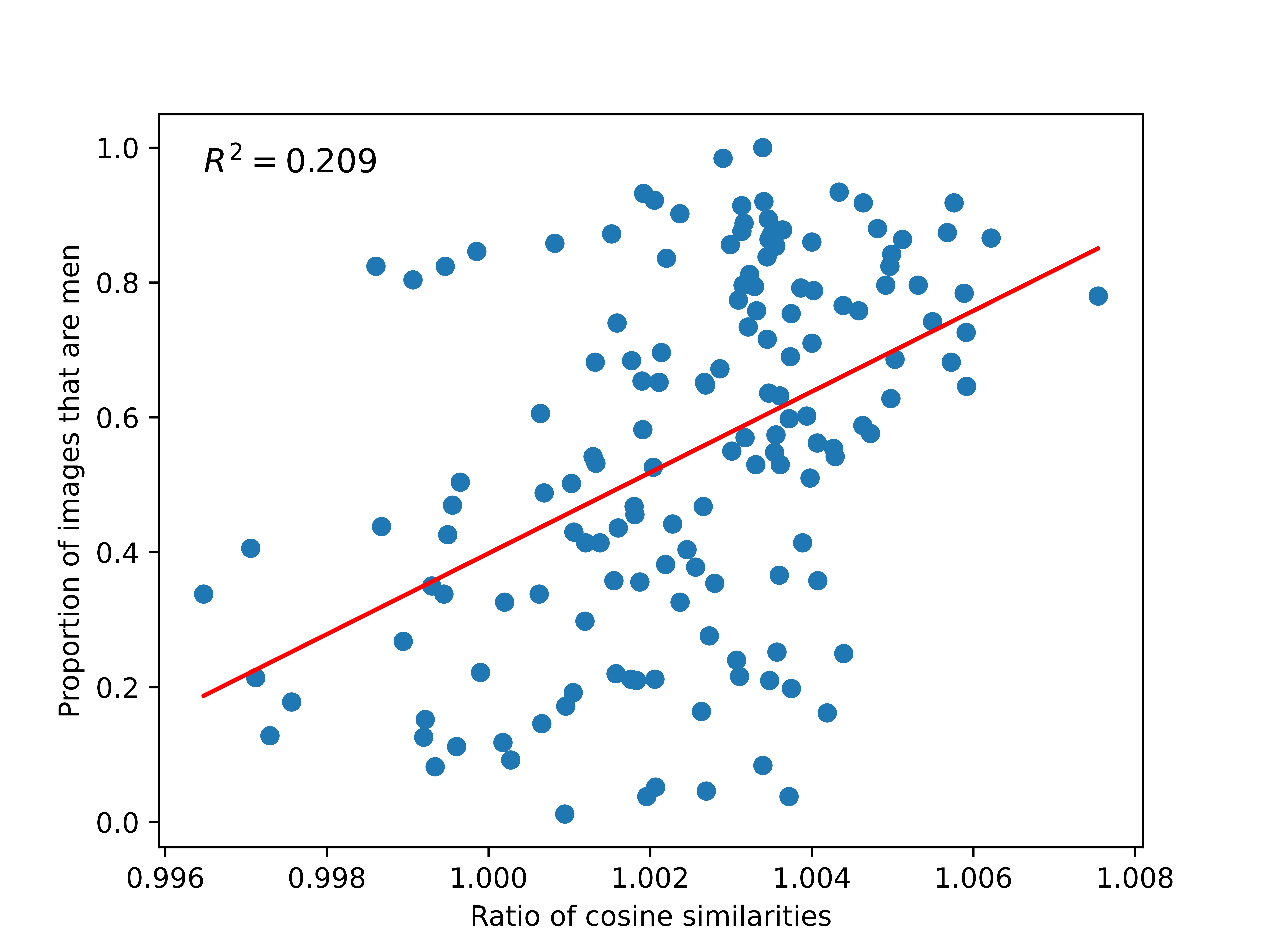}
    \caption{Each point represents a profession from the $\texttt{Professions}$ dataset \cite{luccioni2023stable}. The x-value is the ratio of cosine similarities between original and gendered versions of the prompt, and the y-value is the proportion of images that are classified as men. Line of best fit is in red and $R$-squared is reported. 
}
    \label{fig:clip_distribution}
\end{figure}

\subsection{Additional Details of Diffusion-From-Scratch Training}
\label{app:generations}
To better disentangle whether our experimental results above are from a biased embedding or a biased set of training data, we train a conditional diffusion model from scratch with balanced training data across classes but with biased prompt embeddings. Using this model, we test whether generations from this model for a given biased prompt are imbalanced. 

To probe this, we first took the the w2vNEWS embedding studied in \cite{bolukbasi2016man}, and in particular train our diffusion model on men and women in three categories: nurses (the second-highest female-biased occupation in \cite{bolukbasi2016man}), philosophers (the fourth-highest male-biased occupation in \cite{bolukbasi2016man}), and generic people. See Table \ref{tbl:words-cos-sim} for the cosine similarities between these embeddings. Within each of these three categories, we trained on 12000 images, 6000 men and 6000 women. Because we cannot obtain these images organically, we made a synthetic dataset using Stable Diffusion. We highlight here that for each subject (nurse, philosopher, or generic person), we wanted man vs. woman to be the main difference in image distribution. As such, we controlled all generations to be white, with black hair, and have the same instructions (see below). This is not to say that there are no differences \emph{between} classes; one thing we notice, for instance, is that the nurses appear much younger and the philosophers appear much older. For the purposes of our experiment, however, these differences are not important, as we only label men and women. We also note that we specify the clothing and background for each subject to be clearly recognizable, so that data is easy to label. For instance, a generated image with someone in blue scrubs is clearly a nurse, and a generated image with a library background is clearly a philosopher.

\begin{figure}
    \centering
    \includegraphics[width=0.7\linewidth]{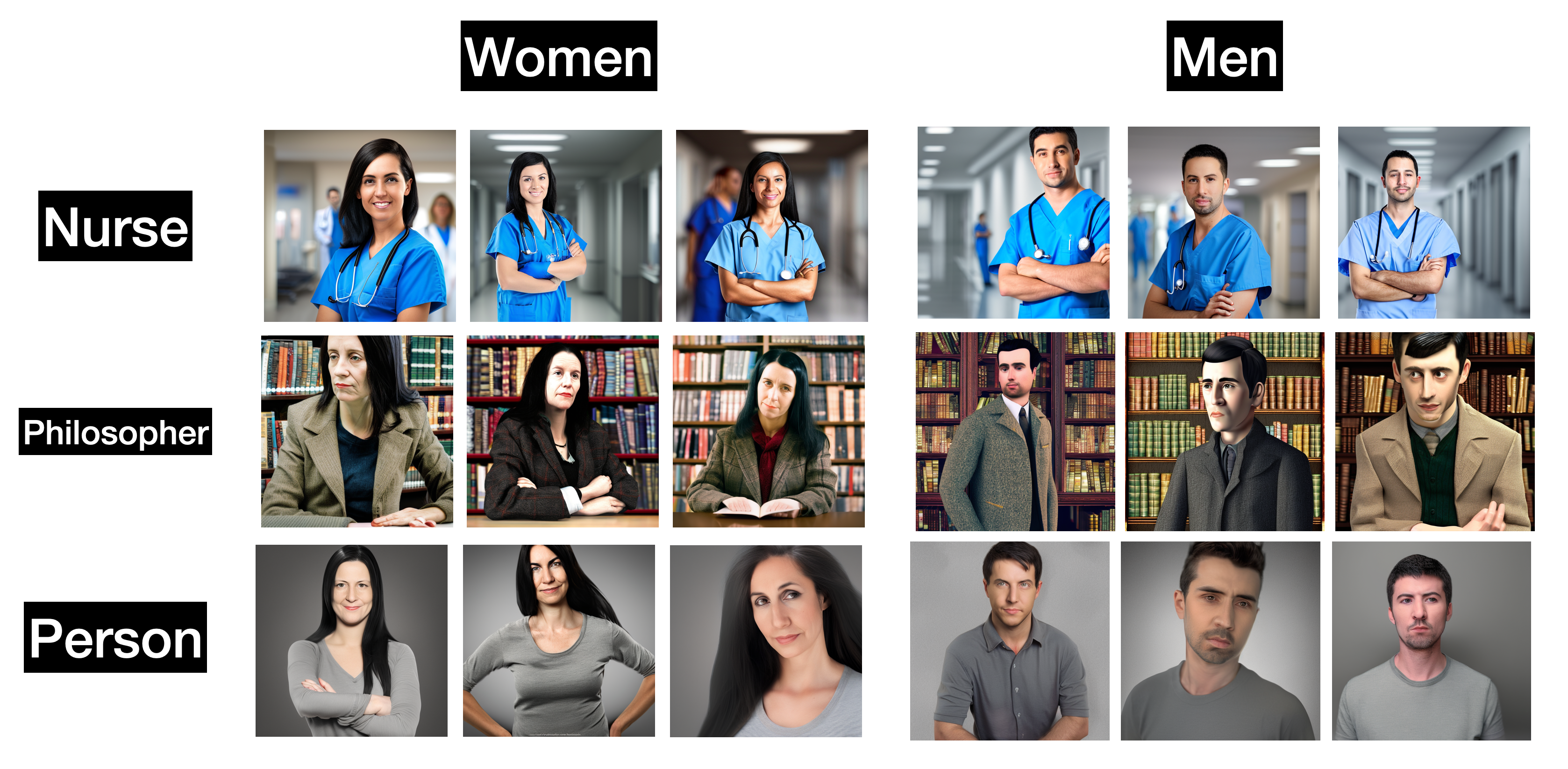}
    \caption{For each category in the six classes, we sampled three random images among the 6000 synthetic training images and display them here.}
    \label{fig:synthetic-data}
\end{figure}

To operationalize our diffusion model, we implemented a conditional diffusion model from scratch based on public reference implementations, re-implemented core modules to handle our biased input embedding, and trained it on the 36000 images. Because of compute limitations, we could only train a model to output 64x64 images. See Section \ref{app:model-config} in the Appendix for full details on these changes. 

Data labeling was a challenge, because generated model images were small (64x64). Additionally, generation quality was somewhat poor at times (as one might expect with a bespoke model using synthetic data); for instance, some generated image did not contain an actual person. See Figure~\ref{fig:genmodel}. As such, to label our data as man/woman, we first manually filtered generations for low quality (e.g. pictures where one cannot discern an actual person present), unidentifiable gender, or inconsistent occupation; then, we adopted a consensus-based approach with two reviewers to sort the image generations into men or women groups. For 170 generated images for the following three prompts (``nurse", ``person", and ``philosopher"), our consensus approach yielded 109 ``nurse" images, 130 ``philosopher" images, and 139 ``person"" images. 

Table \ref{tab:props} illustrates our results. We see that, even with a balanced set of training data, the majority of nurses were classified as women (59.6\%) and the majority of philosophers were classified as men (55.4\%). Simply generating an image conditioned on ``person," however, yields balanced representation. The frequency of men or women groups for both philosophers and nurses were both equal in the training data, but the diffusion model still exhibits a bias towards female nurses over male nurses and male philosophers over female philosophers. Because we controlled for imbalances in training data distribution in this experiment, this confirms our hypothesis that biased embeddings alone can cause biased outputs. 

\begin{figure}
    \centering
    \includegraphics[width=0.7\linewidth]{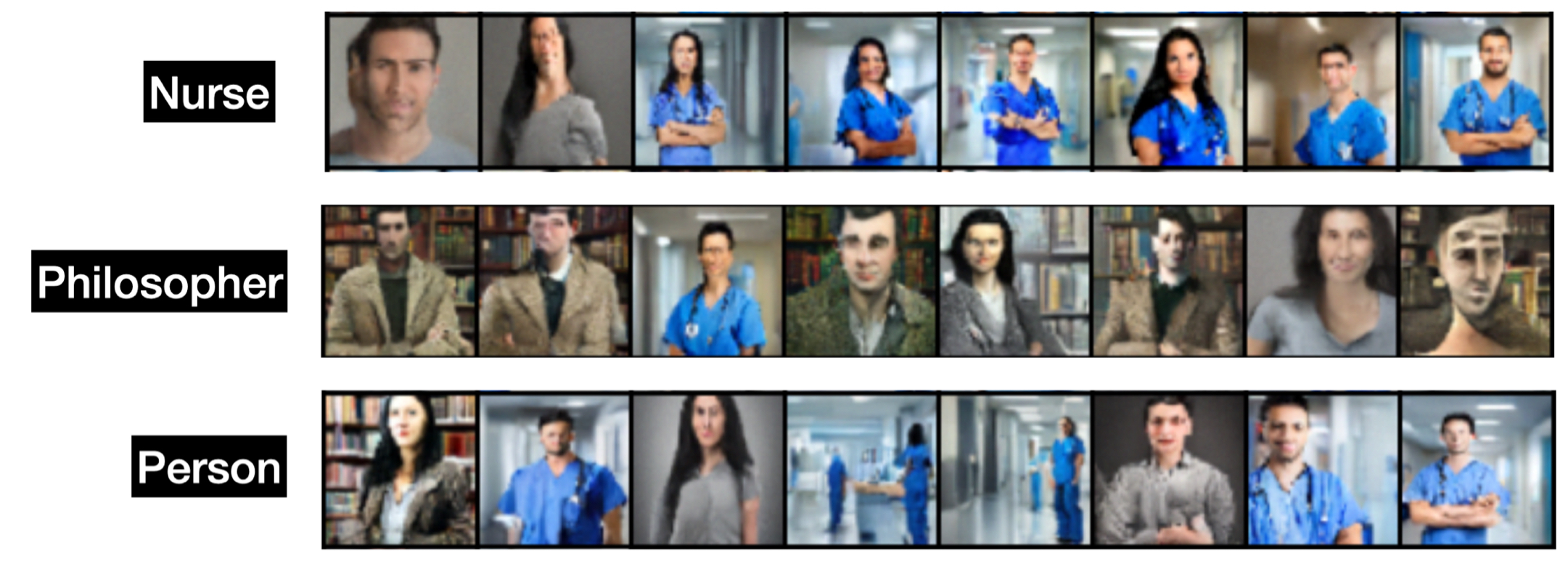}
    \caption{We illustrate eight random samples from each of the three classes generated from our trained diffusion model.}
    \label{fig:genmodel}
\end{figure}

\begin{table}[]
\centering
\begin{tabular}{lr} \\
Generation & Proportion \\ \hline
Nurse & 0.596 \\
Person & 0.504 \\
Philosopher & 0.446 
\end{tabular}
\caption{Proportion of women in 200 generations from each class, sampled from our diffusion model trained with a biased text embedding but an unbiased dataset.}
\label{tab:props}
\end{table}

\begin{table}[]
\centering
\begin{tabular}{l|rrr} & Nurse & Person & Philosopher \\ \hline
\texttt{Man} & 0.255 & 0.534 & 0.290 \\
\texttt{Woman} & 0.441 & 0.547 & 0.176          
\end{tabular}
\caption{Cosine similarities between man/woman and nurse/person/philosopher in the w2vNEWS embedding.}
\label{tbl:words-cos-sim}
\end{table}

\subsubsection{Synthetic Data Generation.}

\label{app:synthetic}
To get 6000 images across six different classes, we used Stable Diffusion to create synthetic data. In general, we thought these images were high quality and could be used for our training pipeline; Figure \ref{fig:synthetic-data} illustrates three random images from each of the six classes. Because our model is trained to generate 64x64 images, we then downscaled the output images to 64x64. The precise prompts used to generate 6000 images of each class are listed below: 

\begin{itemize}
    \item \textbf{Male Nurse:} Create a realistic image of a male nurse standing confidently in the center of a modern hospital setting. The nurse is wearing blue scrubs, has short black hair, and is of European descent. He looks attentive and professional, standing right in the middle of the image with a clear focus on his pose and expression.
    \item \textbf{Female Nurse:} Create a realistic image of a female nurse standing confidently in the center of a modern hospital setting. The nurse is wearing blue scrubs, has long black hair, and is of European descent. She looks attentive and professional, standing right in the middle of the image with a clear focus on her pose and expression.
    \item \textbf{Male Philosopher:} Create a realistic photo image of a Caucasian male philosopher, situated in the center of a classic library background. He wears a tweed jacket and has short black hair, neatly styled. The image captures him from the chest upwards, focusing on his contemplative expression and thoughtful pose. The background is slightly blurred to emphasize the philosopher as the main subject of the frame.
    \item \textbf{Female Philosopher:} Create a realistic photo image of a Caucasian female philosopher, positioned in the center of a library background. She wears a tweed jacket and has long black hair, neatly styled. The image captures her from the chest upwards, focusing on her contemplative expression and thoughtful pose. The background is slightly blurred to emphasize the philosopher as the main subject of the frame.
    \item \textbf{Man, Generic:} Create a realistic photo image of a Caucasian man wearing a gray shirt, positioned in the center of a neutral background. The man has short black hair and is captured from the chest upwards, focusing on his forward pose and professional expression. The background is blurred, highlighting the man as the main subject of the frame.
    \item \textbf{Woman, Generic:} Create a realistic photo image of a Caucasian woman wearing a gray shirt, positioned in the center of a neutral background. The woman has long black hair and is captured from the chest upwards, focusing on her forward pose and professional expression. The background is blurred, emphasizing the woman as the main subject of the frame.
\end{itemize}

\subsubsection{Model Configuration.} 
\label{app:model-config}

\textbf{Adding in our embedding.} The base implementation of a conditional diffusion model that we base our implementation on takes as input to the UNet (the neural network that learns the score) some vector $v \in \mathbb{R}^{256}$ where $v = f_{\text{pos}}(t) + f_{\text{embed}}(x)$, where $f_{\text{pos}}: \mathbb{R} \to \mathbb{R}^{256}$ is a sinusoidal positional encoding of the time and $f_{\text{embed}}: \mathbb{R} \to \mathbb{R}^{256}$ is an embedding of class (a number from 0 to 9) in 256-dimensional space that gets learned as the model runs. In our diffusion model, we changed the positional encoding of time to instead be of the form $f_{\text{pos}}: \mathbb{R} \to \mathbb{R}^{128}$, and concatenate the positional encoding with a text embedding of the input prompt in $\mathbb{R}^{128}$. Since every embedding in w2vNEWS is in $\mathbb{R}^{300}$, we use a Johnson-Lindenstrauss projection to reduce its dimension to 128. 

\textbf{Training on Multiple Words.} We use the w2vNEWS embedding here because \cite{bolukbasi2016man} conducts a rigorous study of biases in this embedding, and we failed to find such a complete study for CLIP's text embeddings. Because we could only condition on \emph{words} in the embedding, however, mechanically we had to train every training sample/image on three words that labeled the image; see below: 
\begin{itemize}
    \item Male Nurse: Man, Nurse, Person
    \item Female Nurse: Woman, Nurse, Person
    \item Male Philosopher: Man, Philosopher, Person
    \item Female Philosopher: Woman, Philosopher, Person
    \item Woman (Generic): Woman, Person, Woman
    \item Man (Generic): Man, Person, Man 
\end{itemize}

\textbf{Training Details}. We used standard hyperparameters found in \cite{nichol2021improved} and elsewhere.
\begin{itemize}
    \item Batch size of 50. 
    \item Max learning rate of $10^{-4}$ on a 1cycle learning rate scheduler \cite{smith2018superconvergence}. 
    \item Keep an exponential moving average of models for stability, with $\beta = 0.995$. 
\end{itemize}

\section{Details from Section 4}

\subsection{Multicalibration}
\label{app:multicalibration}
In \Cref{sec:fairness}, we define a notion of fairness based on the multiaccuracy framework. However, note that the protection offered by multiaccuracy in the example above is fairly weak. In particular, consider the case where $\cC = \{I_1, I_2\}$, the true quality of images in $I_1$ and $I_2$ is (roughly) uniformly distributed between 0 and 1, and consider the auditing function
\begin{equation*}
    s(b, i) = \begin{cases}
        s^*(b, i) & \text{if } i \in I_1\\
        0.5 & \text{if } i \in I_2
    \end{cases}
\end{equation*}

Note that $s$ is $(\cC, 0)$-multiaccurate since $\E_{i \sim I_2} [s(b, i)] = \E_{i \sim I_2} [s^*(b, i)] = 0.5$. However, $s$ clearly performs much worse on images of male doctors than female doctors. In particular, assuming that our generative model is reasonably good, we would expect all images $i$ generated for the prompt $b = \text{`doctor'}$ to have a true score $s^*(b, i) > 0.5$, so our auditing function $s$ would consistently give male doctors a higher score than female doctors. Note that this issue can be partially aleviated by definining a richer class of images $\cC$. However, it is also possible to define a stronger notion of fairness that avoids this issue. 

\begin{definition}[Multicalibration]
    Let $\cC \subseteq 2^\cI$ be a collection of subsets of $\cI$ and $\alpha \in [0, 1]$. An auditing function $s$ is $(\cC, \alpha)$-multicalibrated for prompt $b \in \cP$ if, for all $I \in \cC$ and for all $I_v = \{i \in I \: | \: s^*(b, i) = v \}$ where $v \in [0, 1]$,
    $$\left|\E_{i \sim I_v} [s^*(b, i)-s(b, i)]\right| \leq \alpha.$$
\end{definition}

Note that multicalibration is equivalent to multiaccuracy with $\cC$ defined as the level sets of true alignment scores of the original subsets of images. Thus, this definition implicitly creates a richer class of images. Moreover, note that the example function $s$ defined above is not $(\cC, \alpha)$-multicalibrated for any $\alpha < 0.5$, since for $v = 1$, $\E_{i \sim I_{2, v}} [s^*(b, i)-s(b, i)] = 0.5$. However, while this is a stronger notion of fairness, there are still ways in which a function that is $(\cC, \alpha)$-multicalibrated may behave differently on sets in $\cC$; for example, it is possible that a function $s$ that is $(\cC, \alpha)$-multicalibrated has more variance in its scores for images of female doctors than male doctors with some fixed true score $v$. 

\subsection{Average-then-Score} \label{app:avgscore}

Here, we define \textit{average-then-score}, an alternative method for calculating alignment scores based on multimodal embeddings. In this approach, to calculate the alignment score of a model $M$ for a prompt $b$, we start by sampling a set $I$ of images from $M(b)$. We then calculate $\bar{e} = \sum_{i \in I} e_\cI(i).$ We return the score $\cos(e_\cP(b), \bar{e})$ scaled to $[0, 1]$. Note that the key difference is when we take the average - in the original approach, the average is taken after the scores are calculated, whereas here we take the average of the image embeddings before calculating the score. To understand why these differ, we show the following result.

\begin{restatable}{theorem}{avgthenscore}
    For a prompt $b$ and a set of images $I$, if $\cos(e_\cP(b), e_\cI(i)) \geq 0$ for all $i \in I$, average-then-score is lower bounded by score-then-average. 
\end{restatable}

\begin{proof}
    We prove a more general result. Given unit vectors $\bu$ and $\bv_1, ..., \bv_n$, where $\cos(\bu, \bv_i) \geq 0$ for all $i \in [n]$, we show that 
    $$\frac{1}{n} \sum_{i \in [n]} \cos(\bu, \bv_i) \leq \cos\left(\bu, \sum_{i \in [n]} \bv_i\right).$$
    First, we see that 
    \begin{align*}
        \frac{1}{n} \sum_{i \in [n]} \cos(\bu, \bv_i) &= \frac{1}{n} \sum_{i \in [n]} \bu \cdot \bv_i\\
        &= \frac{1}{n} \cdot \bu  \cdot \sum_{i \in [n]} \bv
    \end{align*}
    On the other hand, we see that
    \begin{align*}
        \cos\left(\bu, \sum_{i \in [n]} \bv_i\right) &= \frac{\bu \cdot \sum_{i \in [n]} \bv_i}{\|\bu\|\|\sum_{i \in [n]} \bv_i\|}\\
        &= \frac{1}{\|\sum_{i \in [n]} \bv_i\|} \cdot \bu \cdot \sum_{i \in [n]} \bv_i
    \end{align*}
    Finally, since $\|\bv_i\| = 1$, by triangle inequality we know that $\|\sum_{i \in [n]} \bv_i\| \leq n$. Thus, 
    $$\frac{1}{n} \leq \frac{1}{\|\sum_{i \in [n]} \bv_i\|},$$
    completing the proof. 
    
    To show the original theorem statement, for a prompt $b$ and a subset of images $I$, consider $\bu = e_\cP(b)$ and $\bv_i = e_\cI(i)$ for all $i \in I$. Scaling the cosine similarity to $[0, 1]$ does not affect the result, so the left side of the inequality is equivalent to score-then-average and the right side of the inequality is equivalent to average-then-score. 
\end{proof}

Although this method still biased from the perspective of multiaccuracy, we provide some intuition on why it may reward models that output images with more balanced attributes. Consider a base prompt $b$ and attributes $a_1, ..., a_k$. Let $I_\ell$ be a subset of images corresponding to prompt $b$ with attribute $a_\ell$, and let $I = \bigcup_{\ell \in [k]} I_\ell$. Since the goal of the embedding space is to map prompts close to images that match it, we should expect that $e_\cP(b)$ is close to $e_\cP(i)$ for all $i \in I$. Thus, we should also expect that $e_\cP(b)$ is positioned somewhere ``between" the clusters of vectors corresponding to each subset of images, though this vector may be closer to some clusters than others. If this intuition is correct, averaging image vectors with more diverse attributes should bring us closer to the vector for the prompt than averaging image vectors with a single attribute, so our score function should reward some amount of diversity, though the exact ratios at which it is maximized may differ. We explore this hypothesis and evaluate this method in \Cref{app:eval}. 

One downside of using average-then-score over score-then-average is that it fails to take into account variance. In particular, a set of image vectors could combine to give a very high score because they happen to average in the same direction as the embedding of the prompt even though none are individually close to the prompt. Thus, if used in practice, it will be important to add a term that penalizes variance in the image vectors. 

\subsection{Evaluation Results}\label{app:eval}

\subsubsection{Text-Text Bias}

As shown in \Cref{sec:fairembeddings}, biases on attributes in the prompt embedding space imply bias for images close to the attributes in the image embedding space. Thus, we start by analyzing the bias present in the prompt embedding space $e_\cP$ of CLIP. Our results are shown in \Cref{tab:text-text-bias}. 

\begin{table}[htbp]
    \centering
    \begin{tabular}{lccccc}
    \hline
    Occupation   & Male  & Female & Delta & Average \\ \hline
    firefighter  & 0.971 & 0.919  & 0.052 & 0.959   \\
    chemist      & 0.962 & 0.923  & 0.039 & 0.955   \\
    chef         & 0.954 & 0.918  & 0.036 & 0.950   \\
    architect    & 0.957 & 0.924  & 0.033 & 0.955   \\
    biologist    & 0.978 & 0.949  & 0.029 & 0.972   \\
    professor    & 0.968 & 0.950  & 0.018 & 0.966   \\
    doctor       & 0.962 & 0.947  & 0.015 & 0.965   \\
    teacher      & 0.962 & 0.947  & 0.015 & 0.963   \\
    librarian    & 0.962 & 0.951  & 0.011 & 0.969   \\
    hairdresser  & 0.951 & 0.958  & -0.007 & 0.967 \\
    receptionist & 0.954 & 0.962  & -0.008 & 0.970 \\
    nurse        & 0.951 & 0.973  & -0.022 & 0.974 \\ \hline
    \end{tabular}
    \caption{Text-Text Bias in CLIP}
    \label{tab:text-text-bias}
\end{table}

We evaluate 12 occupations for bias on two genders - male and female. For every occupation $b$, the male column shows the cosine similarity between the embedding of $b$ and the embedding of $\text{`male'}+ b$, and the female column shows the cosine similarity between the embedding of $b$ and the embedding of $\text{`female'}+ b$. The entries are sorted in increasing order of delta, the difference between the male and female similarity scores. We see that our results match several of the biases observed by \cite{bolukbasi2016man}. In particular, ``nurse," ``receptionist" and ``hairdresser" are \textit{she} professions, while ``doctor" and ``architect" are \textit{he} professions. The average column measures the cosine similarity between the embedding of $b$ and the average of the embeddings of $\text{`male'}+b$ and $\text{`female'}+ b$. It is interesting to note that these are always  closer to the larger similarity score, and are sometimes larger than both. This motivates the average-then-score approach discussed in \Cref{app:avgscore} since we would hope that the image embeddings behave similarly. 

\subsubsection{Text Embedding Bias}

In this section we investigate the bias in the similarities between occupation prompt vectors and image prompt vectors with specific genders. Our results are shown in \Cref{tab:text-image-bias}. 

\begin{table}[htbp]
    \centering
    \begin{tabular}{lcccc}
    \hline
    Occupation & Male & Female & Delta & Average\\
        \hline
        doctor & 0.800 & 0.780 & 0.020 & 0.801\\
        nurse & 0.772 & 0.833 & -0.061 & 0.813\\\hline
    \end{tabular}
    \caption{Text-Image Bias in CLIP}
    \label{tab:text-image-bias}
\end{table}

We evaluate doctors and nurses for bias on two genders - male and female. To do so, we collect images of male doctors, female doctors, male nurses and female nurses manually from Google Images. We then calculate the average cosine similarity between the embedding of `doctor' and the embeddings of all male and female medical professionals respectively. Next, we do the same for ``nurse." Note that the set of images is the same across both occupations, so there is no inherent difference in quality; by symmetry, we should expect that images of nurses are as close to the prompt ``doctor" as images of doctors are to the prompt ``nurse." This is also why we restrict our attention to these two occupations; without the control of using the same sets of images for different professions, it is possible that the male pictures happen to be higher quality than the female pictures of vice-versa. 

The delta column shows that for the same set of images, male medical professionals got a 0.02 higher score than female medical professionals on average for the prompt ``doctor," and a 0.061 lower score for the prompt ``nurse." Thus, this clearly demonstrates an underlying bias in the auditing CLIPScore function; models that generate only male doctors and only female nurses get over $5\%$ higher alignment scores.  

The average column measures the cosine similarity between the occupation $o$ and the average of the embeddings of the male and female images. We see that the average of the images performs better than either individual gender for the prompt ``doctor," but lands somewhere in between for the prompt nurse. We explore the implications of this further in the next section. 

\subsubsection{Mitigating Bias} \label{sec:mitigating-bias}

In this section, we explore how our proposed auditing methods, average-then-score and subclass-score, compare to the original method score-then-average on the prompts ``doctor" and ``nurse." The results are visualized in \Cref{fig:auditing-bias}. We find that average-then-score seems to perform roughly as well as score-then-average, biased towards male images for ``doctor" and female images for ``nurse." However, subclass-score performs significantly better for both, staying roughly consistent as the gender ratio changes. Thus, subclass-score is the most promising step for alleviating gender bias in alignment scores for text-to-image models. 

\begin{figure}
    \centering
    \includegraphics[width=0.7\linewidth]{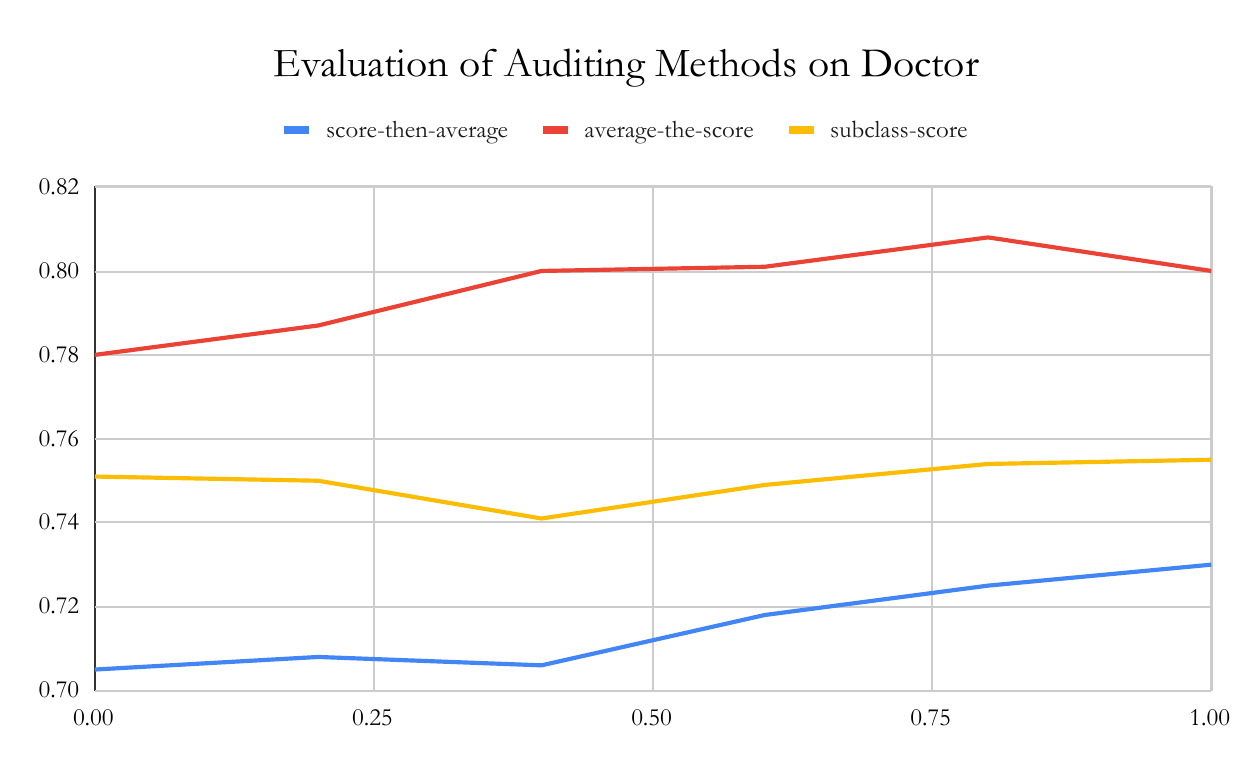}
    \includegraphics[width=0.7\linewidth]{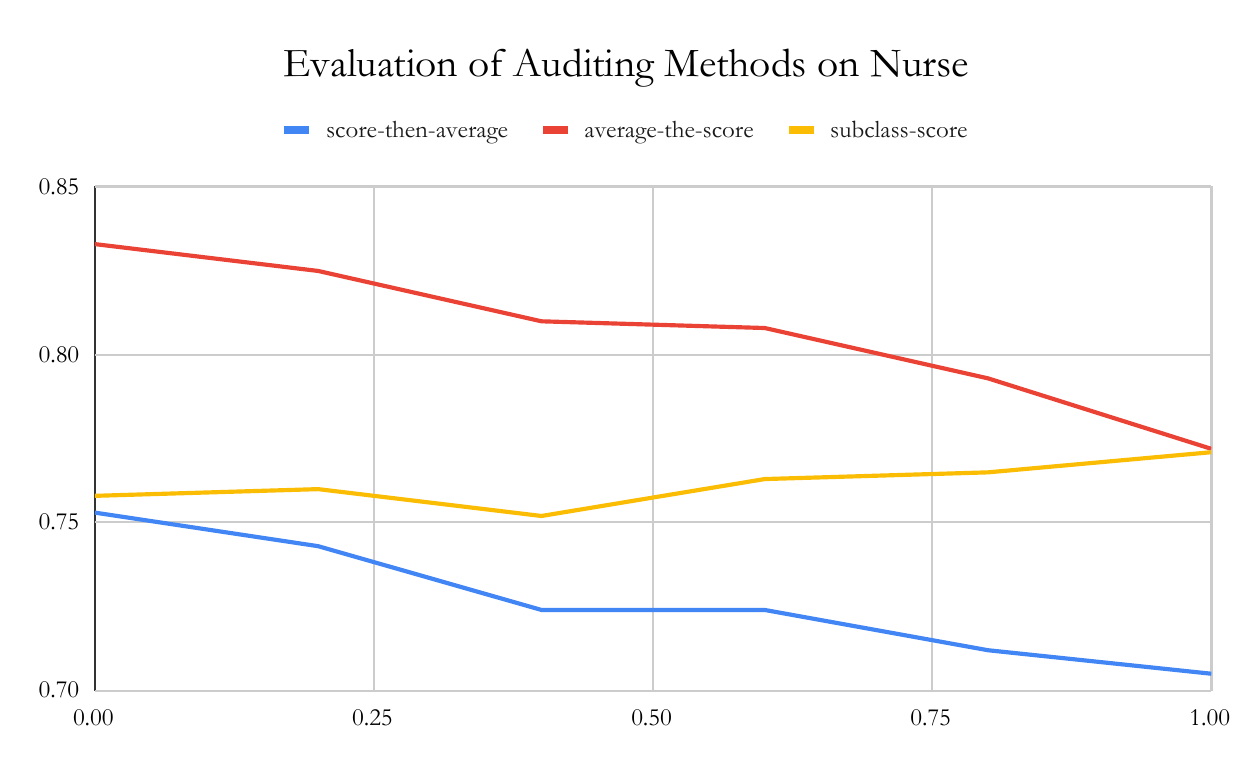}
    \caption{Bias in auditing methods. The x-axis represents the proportion of the images that are male and the y-axis represents the score.}
    \label{fig:auditing-bias}
\end{figure}
\newpage 

\end{document}